\documentclass[11pt]{article}
\usepackage{mathrsfs}
\usepackage{amsmath}
\usepackage{amsfonts}
\usepackage{amssymb}
\usepackage{latexsym}
\usepackage[dvips]{graphicx}
\usepackage{hyperref}
\newtheorem{theorem}{\bf Theorem}[section]
\newtheorem{lemma}[theorem]{\bf Lemma}

\newenvironment{proof}{\noindent{\em Proof:}}{\quad \hfill$\Box$\vspace{2ex}}

\def \and {\, \mbox{\rm and}\, }

\makeatletter

\newcommand{\Rmnum}[1]{\expandafter\@slowromancap\romannumeral #1@}
\makeatother
\begin{document}
\title{\bf On the robust learning mixtures of linear regressions}

\author{ Ying Huang\thanks{School of Digital Economics, Jiangxi Vocational College of Finance and Economis,
 P. R. China. E-mail address: {\it 864694186@qq.com}.}  ,    Liang Chen\thanks{Department of Mathematics, Jiujiang University,
 P. R. China. E-mail address: {\it chenliang3@mail2.sysu.edu.cn}. ORCID: https://orcid.org/0000-0003-3750-1071.}}
\date{}
\maketitle

\begin{abstract}
In this note, we consider the problem of robust learning mixtures of linear regressions. We connect mixtures of linear regressions and mixtures of Gaussians with a simple thresholding, so that a quasi-polynomial time algorithm can be obtained under some mild separation condition. This algorithm has significantly better robustness than the previous result. 


\end{abstract}

\section{Introduction}
We consider the following  mixtures of linear regressions (MLR) model
\begin{equation}\label{eq}
z \sim \operatorname{multinomial}(p), x \sim \mathcal{N}\left(0, \mathbf{I}_{d}\right), y=\left\langle w_{z}, x\right\rangle+\eta
\end{equation}
where $p \in \mathbb{R}^{k}$ is the proportion of different components satisfying $\sum_{i=1}^{k} p_{i}=1$, $\mathcal{N}\left(0, \mathbf{I}_{d}\right)$ denotes the $d$-dimensional standard Gaussian distribution, $\eta$ represents the noise obeying zero-mean $\sigma$-subgaussian distribution $\tau_{\sigma}$ (namely, $\mathbf{E}_{\eta \sim \tau_{\sigma}  } e^{t \eta} \leq e^{\sigma^2 t^2 / 2}$ for every $t\in \mathbb{R}$). The goal is to recover the weights $\left\{w_{i}\right\}_{i=1}^{k}$ from a dataset $\left\{\left(x_{\ell}, y_{\ell}\right)\right\}_{\ell=1}^{N}$ which is i.i.d. generated by (\ref{eq}).\\
\textbf{Assumptions.} We make the following assumptions.\\
(A1) For every $i \in[k], p_{i} \geq p_{\min }$ for some $p_{\min }>0$.\\
(A2) Each $c \leq \left\|w_{i}\right\|_{2} \leq 1$, and for some $\Delta \in(c ,1),\left\|w_{i}-w_{j}\right\|_{2} \geq \Delta$ for any $i \neq j \in[k]$, where $c \in (0,1)$ is some constant.

MLR has been concerned by the statistical learning and theoretical computer science communities for a long time \cite{Richard,Michael,Kai,Hanie,Sivaraman,Jason}. These studies have different focuses, such as studying the local convergence of the non-convex algorithm \cite{Susana,Sivaraman,Jeongyeol,Yuanzhi,Xinyang}, and constructing an efficient algorithm under the condition of non-degenerate parameter matrix \cite{Arun,Kai,Hanie}. This paper mainly focuses on the algorithm design under the general setting \cite{Yuanzhi,Sitan,Ilias}, namely, only the separation constant $\Delta$ is required. 
 Recently \cite{Ilias}, proposed a quasi-polynomial time algorithm for MLR, the algorithm depends on a strong algebraic geometry result established by them to extract the information from positive definite tensors. To the best of our knowledge, the time complexity of their algorithm is currently the best in the general setting. 
 However, 
this algorithm has weak anti-noise ability, that is, it requires $\Delta>\sigma \times {\rm poly}(k)$, which is obviously strict for the case of a large number of components. On the other hand, we know from \cite{Ilias} that $\Delta/\sigma$ cannot be arbitrarily small in a noisy environment, otherwise it will lead to an exponential sample complexity. 


In this paper, we hope to find a compromise between the effectiveness, robustness and conditional constraints of the algorithm. We assume that the separation constant $\Delta$ is not arbitrarily small,
then an robust algorithm with a time complexity of $\exp(\mathcal{O}(((1+\sigma)/c^2 )^2))\times{\rm poly}(1/p_{\min },k,d,1/\epsilon, (dk) ^{(\log (1/\epsilon))^2})$ can be obtained. This is a quasi-polynomial time algorithm for $1/p_{\min },k,d,1/\epsilon $ as long as the constants $c,1/\sigma $ are not arbitrarily small, where $\epsilon\in (0,1/k)$ denotes recovery accuracy. Compared with \cite{Ilias}, in the noise-free case, we cannot give a fully quasi-polynomial time algorithm for all parameters. However, in the case where the number of components $k$ is large and $\sigma$ is not very small, our method has a weaker requirement for the separation constant $\Delta$, in other words our method is more robust to noise, even if $\sigma$ is larger than $|w_{i}|,i\in[k]$. The key idea of this method is to connect MLR and mixtures of Gaussians through a thresholding, then we can directly apply the recent nice results about the mixture of Gaussians \cite{Samuel}, the thresholding operation can greatly reduce the influence of noise, and the weights can be recovered with arbitrarily accuracy.



\section{Main results}
Our thresholding method is based on a simple fact, direct calculation leads to the following lemma.

\begin{lemma}\label{l1}
For the weights $\{w_i\}_{i=1}^{k}$, the constant $c$ defined in \textbf{Assumptions}, and $0<\epsilon<1/k$, we have 

\begin{equation}
\bigg|\mathbf{E}_{x \sim N(0, I_{d}), \eta \sim  \tau_{ \sigma } }\big[x\mid x\in V_{i,\eta}\big]   -v_{i} \bigg|\le \mathcal{O}( \epsilon)
\end{equation}
where
 \begin{equation} v_{i}=\frac{ C'(1+\sigma)(\log  (1/\epsilon))^{1/2} w_{i}}{c^2|w_{i}|^2}\end{equation} and \begin{equation} V_{i,\eta}\triangleq\big\{x:    C'(1+\sigma)(\log (1/\epsilon))^{1/2}/c^2 >w_{i} \cdot x+\eta  > C'(1+\sigma)(\log (1/\epsilon))^{1/2}/c^2 -c\epsilon\big\}, i=1,2 \dots k. \end{equation}
where $C'$ is some universal constant defined by Theorem 5.1 in \cite{Samuel}.

\begin{proof}
There is an orthogonal matrix $U$ such that $Uw_{i}=|w_{i}|( 0,\dots,0,1)^{T}$, then
\begin{equation}
 \begin{aligned}
&\bigg|\mathbf{E}_{x \sim N(0, I_{d}), \eta \sim  \tau_{ \sigma } }\big[Ux\mid x\in V_{i,\eta}\big]   -U v_{i} \bigg|\\
=&\bigg|\mathbf{E}_{y \sim N(0, I_{d}), \eta \sim  \tau_{ \sigma } }\big[y \mid y_{d}\in V_{d,i,\eta}\big]   -\frac{ C'(1+\sigma)(\log  (1/\epsilon))^{1/2}}{c^2|w_{i}|}( 0,\dots,0,1)^{T} \bigg|,
 \end{aligned}
\end{equation}
where $Ux=y=(y_{1},y_{2},\dots,y_{d})^{T}$ and \begin{equation}V_{d,i,\eta}=\{y_{d}:C'(1+\sigma)(\log (1/\epsilon))^{1/2}/c^2 >|w_{i}| \cdot y_d+\eta >C' (1+\sigma)(\log (1/\epsilon))^{1/2}/c^2 -c\epsilon
 \}.\end {equation}
Since \begin{equation}
 \begin{aligned}
 & \bigg|\mathbf{E}_{y \sim N(0, I_{d}), \eta \sim  \tau_{ \sigma } }\big[y \mid y_{d}\in V_{d,i,\eta}\big]  \\
 -& \mathbf{E}_{\eta \sim  \tau_{ \sigma } } \bigg(\frac{ C'(1+\sigma)(\log  (1/\epsilon))^{1/2}}{c^2|w_{i}|}( 0,\dots,0,1)^{T}-( 0,\dots,0,\eta)^{T}/|w_{i}|\bigg)\bigg|\le \mathcal{O}(c\epsilon/|w_{i}|),
 \end{aligned}
\end{equation}
and $\tau_{\sigma}$ is the zero-mean distribution, we finish the proof.

\end{proof}

\end{lemma}
For each $i$, let us consider $d$-dimensional random variables
\begin{equation}\label {22}\{x=(x_{1},\dots,x_{d}):  x \sim  \mathcal{N}(0, I_{d}), x\in V_{i,\eta},\eta\sim \tau_{\sigma}  \} .\end{equation}
Denote $\mathcal{D}_{i}$ the conditional distribution of the above $d$-dimensional random variables, and \begin{equation}\mu_{i}\triangleq\mathbf{E}_{x \sim N(0, I_{d}), \eta \sim  \tau_{ \sigma } }\big[x\mid x\in V_{i,\eta}\big]=\mathbf{E}_{x \sim \mathcal{D}_{i}}x\end{equation} is the expectation of $\mathcal{D}_{i}$. Lemma \ref{l1} shows that the conditional expectation $\mu_{i}$   is close to the weights (multiply by a factor) we need to recover. What we want to do is to use the result in \cite{Samuel} to cluster the data from the mixed distribution $\sum_{i=1}^{k}p_{i}\mathcal{D}_{i}$ and extract the mean $\mu_{i}$ (or $v_{i}$).  According to Definition 3.2  in \cite{Samuel}, we will prove that the mixed distribution $\sum_{i=1}^{k}p_{i}\mathcal{D}_{i}$ is a $\mathcal{O}(\log(1/\epsilon))$-explicitly bounded mixture model with separation $\mathcal{O}(\frac{(1+\sigma)\sqrt{\log (1/\epsilon)}}{c})$. The following estimation shows that the distance between different vectors $v_i$ is greater than $\mathcal{O}(\frac{(1+\sigma)\sqrt{\log (1/\epsilon)}}{c})$ which meets the separation condition.

\begin{lemma}\label{l3} Suppose $0<\epsilon<1/k$, $|w_{j}-w_{i}|\ge c$ and $1\ge |w_{j}|\ge|w_{i}|\ge c$, we have \begin{equation}|v_i-v_j|\ge  \frac{C'(1+\sigma)(\log (1/\epsilon))^{1/2}}{c}.\end{equation}
\end{lemma}
\begin{proof}
Let $w_j=a w_i+b y$ with $y\perp w_{i},  |y|=1$. From $|w_j|\ge|w_{i}|$ and $|w_{j}-w_{i}|\ge c$, we have \begin{equation} (a^2-1 ) |w_i|^2+b^2 \ge 0 \end{equation}
and \begin{equation}|1-a|^2|w_i|^2+|b|^2\ge c^2.\end{equation} Thus,
 \begin{equation}
 \begin{aligned}
 |v_i-v_j| &=\frac{ C'(1+\sigma)(\log (1/\epsilon))^{1/2}}{c^2|w_i|^2 |w_j |^2 } \big||w_j |^2 w_i-\left|w_i\right|^2 w_j \big|\\&=\frac{C'(1+\sigma) (\log (1/\epsilon))^{1/2}}{c^2\left|w_i\right|^2\left|w_j\right|^2}  \big|[ (a^2-1 ) |w_i |^2+b^2+(1-a)|w_{i} |^2 ] w_{i}- |w_{i}|^2 b y \big|\\
& \ge  \frac{ C'(1+\sigma)(\log (1/\epsilon))^{1/2}}{c^2\left|w_i\right|^2\left|w_j\right|^2}  \big|[ (1-a)|w_{i} |^2 ] w_{i}- |w_{i}|^2 b y \big|   \\&\ge  \frac{C'(1+\sigma) c (\log  (1/\epsilon))^{1/2}}{ c^2\left|w_j\right|^2} \ge  \frac{C'(1+\sigma)\sqrt{\log (1/\epsilon)}}{c}       .
\end{aligned}
 \end{equation}
 \end{proof}

Since $C^{'}$ is a universal constant, we will omit the constant related to it in the later discussion. Next, we need to prove that $\mathcal{D}_{i}$ has some concentration property.

\begin{lemma}\label{l2} Let $0<\epsilon<1/k$, for each $i$, $\mathcal{D}_{i}$ is $\mathcal{O}(\log(1/\epsilon))$-explicitly bounded with variance proxy $ \mathcal{O}((\frac{1+\sigma}{c})^2)  $ (see Definition 3.1 in \cite{Samuel}), namely, for every even $  4\le s\le O(\log(1/\epsilon))$,
the polynomial \begin{equation}p(u)=(C^{''}(1+\sigma)^{2} s/c^2)^{s / 2}\|u\|^s-\mathbf{E}_{x \sim \mathcal{D}_i}\langle(x-\mu_{i}), u\rangle^s \end{equation} should be a sum-of-squares of degree $s$, where $C^{''}$ is some universal constant.

\end{lemma}

\begin{proof}
Without loss of generality, let us assume $w_{i}=|w_{i}|( 0,\dots,0,1)^{T}$,   $\ |w_{i}|\ge c,$ then Equ. (\ref{22}) can be written as follows \begin{equation} \{x=(x_{1},\dots,x_{d}):  x_{d}\in V_{d,i,\eta}, \eta\sim\tau_{\sigma}, x_{j} \sim  \mathcal{N}(0, I ),  j=1,2,\dots,d     \} ,\end{equation}
where \begin{equation}V_{d,i,\eta}=\{x_{d}:C'(1+\sigma)(\log (1/\epsilon))^{1/2}/c^2 >|w_{i}| \cdot x_d+\eta >C' (1+\sigma)(\log (1/\epsilon))^{1/2}/c^2 -c\epsilon
 \}.
\end{equation}
Let $\mu_{d,i}:=\mathbf{E}_{x_d \sim \mathcal{N}(0,I), x_{d}\in V_{d,i,\eta}} x_{d}$.
For $4\le s\le \mathcal{O}(\log(1/\epsilon))$ and $1 \le j\le d-1$, we have \begin{equation}\label{33}\mathbf{E}_{x_j \sim  \mathcal{N}(0,I)}|x_j|^s\le (s/2)^{s/2}, j=1,2,\dots,d-1\end{equation} and
\begin{equation}\label{44}\mathbf{E}_{x_d \sim \mathcal{N}(0,I), x_{d}\in V_{d,i,\eta}} |x_{d}-\mu_{d,i}|^s\le (1+\mathcal{O}(\epsilon))^s \max_{4\le l\le s}\mathbf{E}_{\eta\sim\tau_{\sigma}}|\eta/|w_i||^l \le \mathcal{O}(((1+\sigma)/c)^s (s/2)^{s/2}).
\end{equation}
 Combining Equ.(\ref{33}) and Equ.(\ref{44}), we finish the proof by using Lemma 3.5 in \cite{Samuel}.

\end{proof}

Lemma \ref{l1}, Lemma \ref{l3} and Lemma \ref{l2} show that $\mathcal{D}_{i},i=1,2,\dots k$ conforms the conditions in Theorem 5.1 in \cite{Samuel}.
For the samples $S_{i}$ drawn from $\mathcal{D}_{i}$, using Theorem 5.1 and Theorem 5.11 in \cite{Samuel}, there is a $${\rm poly}(1/p_{\min },k,d,1/\epsilon,1/\delta, (dk) ^{(\log (1/\epsilon))^2})$$ time algorithm outputs the clusters $\mathcal{C}_i,i=1,2,\dots,k$ such that each ${C}_i$ is close to $S_{i}$ and \begin{equation}\bigg|\frac{1}{|\mathcal{C}_i|}\sum_{x\in \mathcal{C}_i}x-v_{i}\bigg|\le \mathcal{O}( \epsilon) \end{equation} with $1-\delta$ probability, where $v_{i}$ is close to the expectation $\mu_{i}$ (see Lemma \ref{l1}). Then, we can recover $w_{i}^{'}:=\frac{w_{i}}{|w_{i}|}=\frac{v_{i}}{|v_{i}|}$ with $\mathcal{O}( \epsilon/|v_{i}|)$ accuracy. Since the noise $\eta$ follows zero-mean $\sigma$-subgaussian distribution $\tau_{\sigma}$,  \begin{equation}w_{i} \cdot \frac{1}{|\mathcal{C}_i|}\sum_{x\in \mathcal{C}_i}x=\frac{1}{|\mathcal{C}_i|}\sum_{x\in \mathcal{C}_i}y +  \mathcal{O}( \epsilon )\end{equation} with $1-\delta$ probability, where $|\mathcal{C}_i|\ge \mathcal{O}((\sigma (\ln \delta) /\epsilon)^2 ) $, the label $y=w_{i}\cdot x+\eta$. Thus, \begin{equation}\label{ww}|w_{i}| =\frac{\frac{1}{|\mathcal{C}_i|}\sum_{x\in \mathcal{C}_i}y }{w_{i}^{'} \cdot\frac{1}{|\mathcal{C}_i|}\sum_{x\in \mathcal{C}_i}x}+  \mathcal{O}( \epsilon ).\end{equation} Since all samples are drawn from the conditional distribution after thresholding (see Lemma \ref{l1}), \begin{equation}\label{qq} |y|\le \mathcal{O}( (1+\sigma)(\log (1/\epsilon))/c^2+\sigma) \end{equation} with high probability. From Equ.(\ref{ww}) and Equ.(\ref{qq}), we can recover $|w_{i}|$ with $\mathcal{O}(\epsilon  )$ accuracy. Finally, \begin{equation} \exp(\mathcal{O}(((1+\sigma)(\log (1/\epsilon))/c^2+\sigma)^2)){\rm poly}(1/p_{\min },k,d,1/\epsilon,1/\delta, (dk) ^{(\log (1/\epsilon))^2})\end{equation}  samples of the original distribution are needed to generate ${\rm poly}(1/p_{\min },k,d,1/\epsilon,1/\delta, (dk) ^{(\log (1/\epsilon))^2})$ samples of the conditional distribution $\sum_{i=1}^{k}p_{i}\mathcal{D}_{i}$, so we have the following theorem.

\begin{theorem}For $0<\epsilon <1/k$, the samples drawn from the MLR model (\ref{eq}), there is an algorithm with time complexity  $\exp(\mathcal{O}(((1+\sigma)/c^2 )^2)){\rm poly}(1/p_{\min },k,d,1/\epsilon,1/\delta, (dk) ^{(\log (1/\epsilon))^2})$ outputs
$\{w_{i}^{*}\}_{i=1}^{k}$ such that $|w_{i}^{*}-w_{i}|\le \mathcal{O}( \epsilon ),i=1,2,\dots,k.$

\end{theorem}


\section{Conclusion}
We connect mixed linear regression and Gaussian clustering through a simple truncation method, resulting in a quasi-polynomial time algorithm that is both computationally efficient and robust under some mild separation condition. Since the algorithms for Gaussian clustering \cite{Samuel,Pravesh,Diakonikolas} can be used for more general distribution, it is worth further discussion to design an efficient and stable mixed linear regression algorithms under the wider distribution setting.

\section*{Data availability}
No data was used for the research described in the article.

\section*{Conflict of Interest}
No potential conflict of interest was reported by the authors.

\section*{Acknowledgments}
This work was supported by the National Natural Science Foundation of China (11971490).

\end{document}